\documentclass[11pt, a4paper, onecolumn]{article}

\usepackage{lineno,hyperref}
\usepackage[T1]{fontenc}
\usepackage{amsmath}
\usepackage{amssymb}
\usepackage{pdfpages}
\usepackage{graphicx}
\usepackage[algo2e]{algorithm2e} 
\usepackage[ruled]{algorithm}%
\usepackage{algorithmic}
\usepackage{caption}
\usepackage{color}
\usepackage{url} 
\usepackage{multirow}
\usepackage{array}
\usepackage{multicol}
\usepackage{geometry}

\newtheorem{proposition}{Proposition}
\newtheorem{proof}{Proof}
\newtheorem{definition}{Definition}

\date{}

\begin{document}


   \begin{center}
     {\LARGE\sffamily Possibilistic Networks: Parameters Learning from Imprecise Data and Evaluation strategy}\\
      \vspace{2ex}
			\small{
			Maroua Haddad$^{1,2}$, Philippe Leray$^{2}$ and Nahla Ben Amor$^{1}$\\
			LARODEC Laboratory ISG, Universit\'{e} de Tunis, Tunisia$^{1}$.\\
			LINA-UMR CNRS 6241, University of Nantes, France$^{2}$.\\
maroua.haddad@gmail.com, philippe.leray@univ-nantes.fr, nahla.benamor@gmx.fr}
   \end{center}

\begin{abstract}
There has been an ever-increasing interest in multi-disciplinary research on representing and reasoning with imperfect data. Possibilistic networks present one of the powerful frameworks of interest for representing uncertain and imprecise information. This paper covers the problem of their parameters learning from imprecise datasets, i.e., containing multi-valued data. We propose in the first part of this paper a possibilistic networks sampling process. In the second part, we propose a likelihood function which explores the link between random sets theory and possibility theory. This function is then deployed to parametrize possibilistic networks.
\end{abstract}

\section{Introduction}
Possibilistic networks \cite{fonck} are graphical representations of independence relationships between a set of variables described by uncertain and imprecise information. Despite the multitude of research endeavors devoted to applying possibilistic networks in real domains or to propagating information, their learning from data remains a real challenge. Only few works address this problem and existing ones \cite{book-borgelt,sanguesa1998possibilistic} are direct adaptations of Bayesian networks learning methods without any awareness of specificities of the possibilistic framework which made them theoretically unsound. The main limitation of existing works is that they try to learn separately the parameters, i.e. possibility distributions coding variables uncertainty, and the structure i.e. the graph of the possibilistic network. Moreover, existing methods suffer from the lack of an accurate and standard validation procedure.

Working on parameters in the possibilistic framework highlights several difficulties when dealing with the learning task, in particular, when we handle uncertain and imprecise data. This is due to the fact that learning leads commonly to additive assessment while the possibility theory is, by definition, maxitive i.e. the possibility of a disjunction of events is the maximum of the possibilities of each event in this disjunction. Thereby, if we want to learn parameters from data in the possibilistic framework, two steps are primordial: the first one focuses in counting the occurrence of observations in the dataset to estimate non-normalized distributions. While the second aims to approximate the latter by possibility distributions. 

This paper rigorously addresses this problem by first proposition of a new possibilistic networks sampling method used to evaluate learning algorithms in which we control the imprecision degree in the generated datasets. In the final part of this paper, we propose a likelihood function exploring the link between random sets theory (additive) and possibility theory (maxitive) which will be deployed to learn possibilistic networks parameters.
 
This paper is organized as follows: Section \ref{s1} gives a brief introduction to possibility theory and presents possibilistic networks and their learning from data. Section \ref{s2} proposes a possibilistic networks sampling algorithms. Section \ref{s3} defines a new possibilistic likelihood function and proposes a possibilistic networks parameters learning approach. 
\section{Basic concepts and possibilistic networks}
\label{s1}
Possibilistic networks \cite{fonck} represent the possibilistic counterpart of Bayesian networks \cite{book-pearl} in the possibilistic framework coined by Zadeh \cite{Zadeh-1978} and developed by Dubois and Prade \cite{dubois2006possibility,dubois1998possibility}. This section first presents basic notations used throughout the paper and introduces possibility theory. Then, it defines possibilistic networks and discusses existing learning methods. 
\subsection{Basic concepts of possibility theory}

\subsubsection{Notations and definitions}
Let $V=\{X_1,...,X_n\}$ be a set of variables such that $D_i$ denotes the domain of $X_i$ and $x_{ik}$ denotes an instance of $X_i$, i.e. each $x_{ik} \in D_i$ corresponds to a state (a possible value) of $X_i$. The agents knowledge (state set) of $X_i$ can be encoded by a possibility distribution $\pi(X_i)$ corresponding to a mapping from the universe of discourse $D_i$ to the unit interval [0,1]. For any state $x_{ik} \in D_i$, $\pi(x_{ik}) = 1$ means that $x_{ik}$ realization is totally possible $\pi(x_{ik}) = 0$ means that $x_{ik}$ is an impossible state. It is generally assumed that at least one state $x_{ik}$ is totally possible and $\pi$ is then said to be normalized. 

Extreme cases of knowledge are presented by \emph{complete knowledge}, i.e. $\exists x_{ik} \in D_i$ s.t. $\pi(x_{ik})=1$ and $\forall x_{ij} \in D_i$ s.t. $x_{ij} \not= x_{ik},  \pi(x_{ij})=0$ and \emph{total ignorance}, i.e. $ \forall x_{ik} \in D_i, \pi(x_{ik})=1$ (all values in $D_i$ are possible). The definition of a possibility distribution could be generalized to a set of variables $V$ defined on the universe of discourse $\Omega= D_1 \times...\times D_n$ encoded by $\pi$. $\pi$ corresponds to a mapping from $\Omega$ to the unit interval [0,1]. $\omega$ is called interpretation or event and is denoted by a tuple $(x_{1k},...,x_{nl})$.  Given a possibility distribution $\pi$, we can define for any subset $A \subseteq D_i$ two dual measures: possibility measure $\Pi(A) = \underset {x_{ik} \in A} \max\ \pi (x_{ik})$ and necessity measure $N(A)= 1 - \Pi(\bar{A})$ where $\Pi$ assesses at what level $A$ is consistent with our knowledge represented by $\pi$ whereas $N$ evaluates at what level $\bar{A}$ is impossible. 

The particularity of the possibilistic scale is that it can be interpreted in two ways: (i) an ordinal manner which means that possibility degrees reflect only a specific order between possible values. (ii) a numerical way meaning that possibility degrees make sense in
the ranking scale. These two interpretations induce two definitions of possibilistic conditioning which consists in reviewing a possibility distribution by a new certain information $A$, an interpretation of $A \subseteq \Omega$. The \emph{product-based} conditioning is defined as follows:
\begin{equation}
\pi(\omega| A) = \left\{\begin{array}{cc}
\frac{\pi(\omega)}{\Pi(A)} & \text{if} ~ \omega \in A \\0 & \text{otherwise}.
\end{array} \right.
\label{cond1}
\end{equation}
The \emph{min-based} conditioning is defined as follows:
\begin{equation}
\label{cond2} \pi(\omega \mid_m A) = \left \{
\begin {array}{ll}
1 &  \text{si} \ \pi(\omega) = \Pi(A) \mbox { \text{and} } \omega \in A \\
\pi (\omega) & \text{if} \ \pi(\omega) < \Pi(A) \mbox { \text{and} } \omega
\in A \\
 0 & \text{otherwise}. \\
\end {array}
\right.
\end{equation}

\subsubsection{Possibility theory and random sets theory}
One view of possibility theory is to consider a possibility distribution $\pi$ on $X_i$ as a \textit{counter function} of a random set \cite{shafer1976mathematical} pertaining to $D_i$. A random set in $D_i$ is a random variable which takes its values on subsets of $D_i$. More formally, let $D_i$ be a finite domain. A basic probability assignment or mass function is a mapping $m:2^{D_i} \longmapsto [0,1]$ such that $\sum_{A_{ik}\subseteq D_i}(m(A_{ik}))=1$ and $m(\emptyset)=0)$. A set $A_{ik} \subseteq D_i$ such that $m(A_{ik})>0$ is called a focal set.

The possibility degree of an event $x_{ik}$ is the probability of the possibility of the event i.e. the probability of the disjunction of all events (focal sets) in which this event is included \cite{book-borgelt}:
\begin{equation}
\label{def2}
\pi(x_{ik}) = \underset{A_{ik}|x_{ik}\in A_{ik}}{\sum} m(A_{ik})
\end{equation}
A random set is said to be \textit{consistent} if there is at least one element $x_{ik}$ contained in all focal sets $A_{ik}$ and the possibility distribution induced by a consistent random set is, thereby, normalized. Exploring this link between possibility theory and random sets theory has been extensively studied, in particular, in learning tasks, we cite for instance \cite{book-borgelt,joslyn1997measurement}.

\subsubsection{Variable sampling}
\label{echan}
The variable sampling corresponds to the generation of a dataset representative of its possibility distribution. In the numerical interpretation, two approaches \cite{chanas1988single,guyonnet2003hybrid} have been proposed to sample a variable. These methods are based on $\alpha$-$cut$ notion: $\alpha$-$cut(X_i)=\{x_{ik} \in D_i$ s.t. $ \pi(x_{ik}) \geq \alpha\} $ where $\alpha$ is randomly generated from [0,1]. The method proposed by Guyonnet et al. in \cite{guyonnet2003hybrid} focuses on the generation of imprecise data by returning all values of $\alpha$-$cut(X_i)$ for any variable $X_i$. Chanas and Nowakowski proposed another method in \cite{chanas1988single} which is dedicated to the generation of precise data by returning a single value uniformly chosen from $\alpha$-$cut(X_i)$. 

\subsection{Possibilistic networks}
\subsubsection{Definition}
\label{def}
Possibilistic networks \cite{fonck} are the possibilistic counterpart of Bayesian networks \cite{book-pearl,book-jensen} sharing the same \textit{graphical component} i.e. a directed acyclic graph (DAG) which encodes a set of independence relations between $V=\{X_1,...,X_n\}$ where each variable $X_i \in V$ is conditionally independent of its non-descendent given its parents. The \textit{numerical component} substitutes the probabilistic framework by the possibilistic one by assigning a conditional possibility distribution to each node $X_i \in V$ in the context of its parents (denoted by $Pa(X_i)$), i.e. $\pi (X_i|Pa(X_i))$. The two definitions of the possibilistic conditioning lead naturally to two different ways to define possibilistic networks \cite{fonck,book-borgelt}: \emph{product-based} possibilistic networks based on the \emph{product-based} conditioning expressed by Equation \ref{cond1}. These models are theoretically and algorithmically close to Bayesian networks. In fact, these two models share the graphical component, i.e. the DAG and the product operator in the computational process. This is not the case of \emph{min-based} possibilistic networks based on \emph{min-based} conditioning defined by Equation \ref{cond2} that represents a different semantic. 

In both cases, possibilistic networks are a compact representation of possibility distributions. More precisely, the joint possibility distribution could be computed by the possibilistic chain rule expressed as follows:
\begin{equation}
\label{chain} \pi_\otimes(X_1,..., X_n)= \otimes_{i=1..n}\pi(X_i \mid_{\otimes} Pa(X_i))
\end{equation}
where $\otimes$ corresponds to the minimum operator (min) for \emph{min-based} possibilistic networks and to the product operator (*) for \emph{product-based} possibilistic networks.

\subsubsection{Learning from data}
Few attempts have been proposed to learn possibilistic networks from data. In fact, Sang{\"u}esa et al. \cite{sanguesa1998possibilistic} have proposed two hybrid methods handling precise data: the first one learns trees and the second one learns the more general structure of DAGs. Borgelt et al. \cite{book-borgelt} have adapted two methods initially proposed to learn Bayesian networks: K2 \ and maximum weight spanning tree \cite{chow1968approximating} to learn possibilistic networks from imprecise data. These attempts concern mainly the structure learning and ignore parameters learning problem. Indeed, Sang{\"u}esa et al. learn probability distributions and transform them into possibility ones. Borgelt et al. methods estimate a possibility distribution using possibilistic histograms i.e. based of number of occurrence of different values of $X_i$ in the dataset. Let $\mathcal{D}_i = \{  d_i^{(l)} \}$ be a dataset relative to a variable $X_i$, $d_i^{(l)} \in D_i$ (resp. $d_i^{(l)}\subseteq D_i$) if data are precise (resp. imprecise).  The number of occurrences of each $x_{ik} \in D_i$, denoted by $N_{ik}$,  is the number of times $x_{ik}$ appears in $\mathcal{D}_i$: $N_{ik}= \text{card}(\{l \  \text{s.t.} \ x_{ik} \in d_i^{(l)}\})$. The sub-normalized estimation  $\hat{\pi}(x_{ik})$ is expressed by:
\begin{equation}
\label{jos}
\hat{\pi}(x_{ik}) = \frac{N_{ik}}{N}
\end{equation}
where $N$ is the number of observations in $\mathcal{D}_i$. N is equal (resp. lower or equal) to the sum of $N_{ik}$ if data are precise (resp. imprecise). 

Equation \ref{jos} could be defined on a set of variables $X_i,X_j,...X_w$. In this case, $N_{ik}$ becomes $N_{ik,jl,...,wp}= N(\{x_{ik}x_{jl}...x_{wp}\} \subseteq \mathcal{D}_{ijw})$.

\section{Evaluation process for possibilistic networks learning algorithms}
\label{s2}
In the probabilistic case, evaluating Bayesian networks learning algorithms is ensured using the following process: we select an arbitrary Bayesian network either a synthetic one or a gold standard from which we generate a dataset using Forward Sampling algorithm \cite{henrion1986propagating}. Then, we try to recover the initial network using a learning algorithm and we compare the initial network with the learned one.

In \cite{HaddadLA15}, we have proposed to transpose the evaluation strategy proposed in the probabilistic case to the possibilistic one. In what follows, will mainly concentrate on sampling possibilistic networks which consists in generating a dataset representative of their joint distributions. The sampling process constructs a database of N (predefined) observations by instantiating all variables in $V$ w.r.t. their possibility distributions. Obviously, variables are most easily processed w.r.t. a topological order, since this ensures that all parents are instantiated. Instantiating a parentless variable corresponds to computing its $\alpha$-$cut$. Instantiating a conditioned variable corresponds to computing also its $\alpha$-$cut$ given its sampled parents values. This could not be directly applied to conditional possibility distribution which is composed of more than one distribution depending on the number of the values of its sampled parents. So, to instantiate a conditioned variable $X_i$ s.t. $Pa(X_i=A)$, we compute $\alpha$-$cut$ from $\Pi(X_i|Pa(X_i)=A)$, computed as follows:
\begin{equation}
\label{eqqq}
\Pi(X_i|Pa(X_i)=A) = \max_{a_i\in A} \pi(X_i|a_i) \pi(a_i)
\end{equation}

The main limitation of this sampling process is that it generates a particular case of imprecise datasets i.e. obtained data relative to a variable $X_i$ are conditionally consonant with respect to the sampled values of its parents. This is due the fact that the sampling process is based on the $\alpha$-cut notion which returns generally most possible values as observed ones. In what follows, we propose to parametrize this sampling process in order to generate more generic imprecise data by controlling the imprecision degree in generated datasets. In fact, we propose an extension to the sampling process proposed in \cite{HaddadLA15} in which we control the imprecision degree of generated data.

The aim of controlling the imprecision degree in generated datasets is to create different forms of imprecision around the most possible value i.e. varying the values in the dataset but we conserve the most possible combination of $\Omega$. Given an imprecision degree $\theta_{imp}$ and a variable $X_i$ such that the $\alpha$-cut$(X_i)$ presents values returned by the sampling process, we generate all subsets pertaining to this $\alpha$-cut including the most possible value and we assign a probability equal to $\theta_{imp}$ to $\alpha$-cut$(X_i)$ and a probability equal to each subset $S_{X_i}$, $\theta_{imp}^{card(S_{X_i})-1}*(1-\theta_{imp})^{card(\alpha\text{-cut}(X_i))-card(S_{X_i})}$ to remaining subsets. Finally, we sample this probability distribution and we replace $\alpha$-cut$(X_i)$ by the sampled subset in the dataset.

The proposed sampling process is formally described by Algorithm \ref{algo2}. 

\begin{algorithm}
\caption{Sampling process (imprecision control)} \label{algo2}
\begin{algorithmic}
\STATE Input: Possibilistic network
\STATE Output: Observation\\
\Begin{
\STATE \% Process nodes in a topological order
\ForEach {$X_i \in V$}
{\eIf{$X_i$ is parentless}{observation$(X_i)$=$\alpha$-cut$(X_i)$}{Compute $\Pi(X_i|Pa(X_i)=$observed) using Equation \ref{eqqq}\\ observation$(X_i)$= $\alpha$-cut$(X_i)$ from $\Pi(X_i|Pa(X_i)=$observed)}
}
$p(\alpha$-cut$({X_i}))$=$\theta_{imp}$\\
\ForEach {$S_{X_i} \subseteq$ cut} {$p(S_{X_i})=\theta_{imp}^{card(S_{X_i})-1}*(1-\theta_{imp})^{card(\alpha\text{-$cut$}(X_i))-card(S_{X_i})}$}
observation($X_i$)=sample($p$)
\STATE Return observation }
\end{algorithmic}
\end{algorithm}

\section{Parameters learning of possibilistic networks}
\label{s3}

\subsection{New possibilistic likelihood function}
The formulation of our likelihood function is made in two steps: first, we propose a likelihood function defined on random sets. Then, we propose an approximation of this likelihood function which leads to the definition of our possibilistic likelihood. 
\begin{definition}
\label{def3}
Let $G$ be a DAG and $ \{m_1, m_2, ..., m_n\}$ be the parameters relative to $\{X_1, X_2, ..., X_n\}$ to be estimated and $\mathcal{D}_{ij} = \{  d_{ij}^{(l)} \}$ be a dataset relative to a variable $X_i$ and its parents $Pa(X_i)=j$, $d_{ij}^{(l)} \subseteq D_{ij}$. The number of occurrences of each $A_{ik} \subseteq {D_i}$ such that $Pa(X_i)=j$ ($j \subseteq {D_j}$), denoted by $N_{ijk}$,  is the number of times $A_{ijk}$ appears in $\mathcal{D}_{ij}$: $N_{ijk}= \text{card}(\{l \  \text{s.t.} \ A_{ijk} = d_{ij}^{(l)}\})$. We express the likelihood function as follows:
\begin{equation}
\label{L1}
 mL(m,G,\mathcal{D})= \prod_{i=1}^n \prod_{j=1}^{q_i} \prod_{k=1}^{r_i} N_{ijk} \log m_{ijk}
\end{equation}
where $mL$ is expressed by random sets of domains variables i.e. for each $X_i$, $q_i$ is card($2^{Pa(X_i)}$) and $r_i$ is card($2^{D_i}$), $m_{ijk}$ is the parameter to be estimated when $X_i=A_{ik}$ and $Pa(X_i)=j$.
\end{definition}
For numerical stability reasons, we propose the log-likelihood function. Equation \ref{L1} becomes:
\begin{equation}
\label{randomLL}
mLL(m,G,\mathcal{D})= \sum_{i=1}^n \sum_{j=1}^{q_i} \sum_{k=1}^{r_i} N_{ijk} \log m_{ijk}
\end{equation} 
Note that mass functions associated to random sets is a probability distribution, the partial derivative of the $mLL(m,G,\mathcal{D})$ follows the same principle of the partial derivative of the probabilistic likelihood function \cite{neapolitan2004learning} and reaches its maximum in $\hat{m}_{ijk}=\frac{N_{ijk}}{\sum_{k=1}^{r_i} {N_{ijk}}}$.

Note that if mass functions are defined on singletons, i.e, available data are precise, the likelihood function defined in Equation \ref{randomLL} recovers the probabilistic one. However, in the opposite case, computing the likelihood functions is computationally expensive. In fact, a random set relative to a variable $X_i$ is defined on $2^{D_i}$ and its cardinality grows exponentially with the the number of values in $D_i$ \cite{dubois1990consonant}. Consequently, we propose to investigate the link between possibility distributions and mass functions presented in Equation \ref{def2} and to define an approximation of random sets likelihood function, i.e. a possibilistic likelihood expressed by possibility distributions defined on singletons. More formally, we express the possibilistic likelihood function as follows:
\begin{definition}
\label{def4}
Let $G$ be a DAG and $ \{\pi_1, \pi_2, ..., \pi_n\}$ be the parameters relative to $\{X_1, X_2, ..., X_n\}$ to be estimated and $\mathcal{D}_{ij} = \{  d_{ij}^{(l)} \}$ be a dataset relative to a variable $X_i$ and its parents $Pa(X_i)=j$, $d_{ij}^{(l)} \subseteq D_{ij}$.  The number of occurrences of each $x_{ik} \in D_i$ such that such that $Pa(X_i)=j$, denoted by $N_{ijk}$,  is the number of times $x_{ijk}$ appears in $\mathcal{D}_{ij}$: $N_{ijk}= \text{card}(\{l \  \text{s.t.} \ x_{ijk} \subseteq d_{ij}^{(l)}\})$.  We express the possibilistic likelihood as follows:
\begin{equation}
\pi LL(\pi,G,\mathcal{D})= \sum_{i=1}^n \sum_{j=1}^{q_i} \sum_{k=1}^{r_i} N_{ijk} \log \pi_{ijk}
\end{equation}
where for each $X_i$ $q_i$ is $\text{card}(Pa(X_i))$ and $\text{card}(r_i=|D_i)$, $\pi_{ijk}$ is the parameter to be estimated when $X_i=x_{ik}$ and $Pa(X_i)=j$.
\end{definition}

\subsection{Possibilistic-likelihood-based parameters learning algorithm }
In the probabilistic case, learning Bayesian networks parameters is performed satisfying \textit{maximum likelihood} principle \cite{Heckerman1999} which evaluates at what level learned parameters fit the dataset. As far as we know, such a measure has not been proposed in the possibilistic framework. The absence of a learning possibilistic networks parameters method could be justified by the fact that the learning is usually viewed as an objective task i.e. based on computing frequency of observations while possibility theory has been almost based on the subjective opinions. This is to some extent true, especially, when we deal with measurement devices leading to precise observations (one possible value per variable). In this case, probability theory remains the most adequate alternative. However, when measurement devices provide imprecise data and we want to model data as they have been collected i.e. including imprecision due to the physical measurement itself, non-classical uncertainty theories stand out as best alternatives. In our case, we choose to use possibility theory since it is able to offer a natural and simple formal framework representing imprecise and uncertain information. The latter refers to the study of maxitive and minitive set-functions and can be interpreted as an approximation of upper and lower frequentist set probabilities in the presence of imprecise observations and this link will be explored in the following. In fact, we use the possibilistic likelihood in Definition \ref{def4} to learn possibilistic networks parameters.

\begin{proposition}		
Given a DAG, a fixed parameter $\pi_{ijk}$ and an imprecision degree $S_i$ (prefixed value) relative to the variable $X_i$ the maximum possibilistic likelihood estimates are the parameter values that maximize $\pi LL(\pi,G,\mathcal{D})$. We assume that $\sum_{k=1}^{r_i} {\pi_{ijk}}$ is a constant equal to $S_{i}$, $\pi LL(\pi,G,\mathcal{D})$ reaches it maximum in $\hat{\pi}_{ijk}= argmax(\pi LL(\pi,G,\mathcal{D}))=\frac{N_{ijk}}{\sum_{k=1}^{r_i} {N_{ijk}}}*S_{i}$.
\end{proposition}		
\begin{proof}
Let $S_{i}$ be $\sum_{k=1}^{r_i} \pi_{ijk}$. So, the parameters $\pi_{ijk}$ are related by the following formula: $\pi_{ijri}=S_{i}-\sum_{k=1}^{r_i-1} \pi_{ijk}$. Then, $\pi LL(\pi,G\mathcal{D})$ could also be rewritten as follows:

\begin{equation}
\pi LL(\pi,G,\mathcal{D})= \sum_{i=1}^n \sum_{j=1}^{q_i} ((\sum_{k=1}^{r_{i-1}} N_{ijk} \pi_{ijk}) + N_{ijr_i} \log (S_{i}-\sum_{k=1}^{r_i-1} \pi_{ijk}))
\end{equation}

So, its derivative w.r.t a parameter $\pi_{ijk}$ is:
\begin{center}
$\frac{\partial \pi LL(\pi,G,\mathcal{D})}{\partial \pi_{ijk}}= \frac{N_{ijk}}{\pi_{ijk}}=\frac{N_{ijr_i}}{S-\sum_{k=1}^{r_i-1} \pi_{ijk}} =\frac{N_{ijk}}{\pi_{ijk}}-\frac{N_{ijr_i}}{\pi_{ijr_i}} $
\end{center}

So, the value $\hat{\pi}_{ijk}$ of the parameter of $\pi_{ijk}$ maximizing the possibilistic likelihood sets this derivative equal to 0 and satisfies thereby:
\begin{center}  
$\frac{N_{ijk}}{\hat{\pi}_{ijk}}=\frac{N_{ijr_i}}{\hat{\pi}_{ijr_i}} $
\end{center}
We have:
\begin{center}
$\frac{N_{ij1}}{\hat{\pi}_{ij1}}=\frac{N_{ij2}}{\hat{\pi}_{ij2}}=...=\frac{N_{ijr_{i-1}}}{\hat{\pi}_{ijr_{i-1}}}=\frac{N_{ijr_{i}}}{\hat{\pi}_{ijr_{i}}} = \frac{\sum_{k=1}^{r_i} N_{ijk}}{\sum_{k=1}^{r_i} \hat{\pi}_{ijk}}=\frac{\sum_{k=1}^{r_{i}} N_{ijk}}{S_{i}}$
\end{center}

So, $\hat{\pi}_{ijk}=\frac{N_{ijk}}{\sum_{k=1}^{r_i} {N_{ijk}}}$*$S_{i}$.
\end{proof}

Note that $S_i$ corresponds to the imprecision degree relative to a variable $X_i$ and could be fixed by an expert, inferred from the dataset to learn from or based on variables description. To obtain normalized possibility distributions, we divide every obtained distribution by its maximum. This operation will eliminate the effect of the imprecision degree and let us to be objective in the learning task. However, it remains possible to fix an imprecision degree per value of variables of the studied domain. Note that if obtaining possibility distributions are equal to zeros, we add an initial count (1) to all instances $N_{ijk}$ whose number are then added to the total number of instances.

\section{Conclusion}
In this paper, we propose an evaluation strategy to possibilistic networks parameters learning algorithms. A sampling method has been proposed to generate an imprecise dataset from a possibilistic network. In the second part of this paper, we propose a new \emph{product-based} possibilistic networks parameters learning algorithm based on a possibilistic likelihood function exploring the link between random sets theory and possibility theory.  

\bibliographystyle{plain}
\bibliography{bibliography}

\end{document}